\documentclass{article}

\usepackage{arxiv}

\usepackage{amsmath,amssymb,amsfonts}%
\usepackage{amsthm}%
\usepackage{mathrsfs}%
\usepackage{algorithm}%
\usepackage{algorithmicx}%
\usepackage{algpseudocode}%
\usepackage{listings}%
\usepackage{multirow}%

\usepackage[utf8]{inputenc} 
\usepackage[T1]{fontenc}    
\usepackage{hyperref}       
\usepackage{url}            
\usepackage{booktabs}       
\usepackage{amsfonts}       
\usepackage{nicefrac}       
\usepackage{microtype}      
\usepackage{cleveref}       
\usepackage{lipsum}         
\usepackage{graphicx}
\usepackage[numbers]{natbib}
\usepackage{doi}
\usepackage{subcaption}

\newtheorem{theorem}{Theorem}[subsection]
\newtheorem{proposition}[theorem]{Proposition}
\newtheorem{definition}{Definition}[subsection]%
\newtheorem{lemma}{Lemma}[subsection]

\title{From level set evolution to threshold optimization: A grayscale level set framework for image segmentation}

\newif\ifuniqueAffiliation

\ifuniqueAffiliation 
\author{Xingkai Li\\
	School of Mathematics\\
	Harbin Institute of Technology\\
	Harbin, Heilongjiang, 150001, China.\\
	\texttt{24b912021@stu.hit.edu.cn}\\
	\And
	Jiebao Sun\\
	School of Mathematics\\
	Harbin Institute of Technology\\
	Harbin, Heilongjiang, 150001, China.\\
	\texttt{24b912021@stu.hit.edu.cn}\\
	\And
	Fanghui Song\\
	School of Mathematics\\
	Harbin Institute of Technology\\
	Harbin, Heilongjiang, 150001, China.\\
	\texttt{24b912021@stu.hit.edu.cn}\\
	\And
	Zhichang Guo\\
	School of Mathematics\\
	Harbin Institute of Technology\\
	Harbin, Heilongjiang, 150001, China.\\
	\texttt{24b912021@stu.hit.edu.cn}\\
}
\else
\usepackage{authblk}

\author[1]{Xingkai Li}%
\author[1]{Jiebao Sun}
\author[1]{Fanghui Song\thanks{\texttt{Corresponding author. Email:fanghuisong@stu.hit.edu.cn}}}
\author[1]{Zhichang Guo}
\affil[1]{School of Mathematics, Harbin Institute of Technology, Harbin, 150001, Heilongjiang, China.}
\fi

\begin{document}
\maketitle

\begin{abstract}
	The segmentation of multiple degradations has been a challenging problem in the field of image segmentation. Existing level set approaches commonly adopt a length regularization term to constrain the geometric shape of the segmentation contour. However, the introduction of the length term often results in numerical instability and high computational cost. In this paper, we show that the length term is not essential under certain smoothness constraints, and theoretically prove that the presence of the length term affects the property of $|\nabla \phi|=1$. Based on the finding, we define a class of smooth images, construct the grayscale level set, and propose a fast segmentation framework for degraded images, such as heavily noisy images and intensity inhomogeneous images. The framework transforms PDE evolution into one-dimensional threshold search, which has significant advantages in computational speed, especially on large-scale images. Experiments validate the segmentation performance of the proposed framework on various degraded images.
\end{abstract}

\keywords{ Image segmentation \and Regularization-free \and Grayscale level set \and Multiple degeneration}

\section{Introduction}\label{sec1}
Image segmentation is a fundamental problem in computer vision and image analysis, aiming to partition an image into meaningful regions and accurately delineate object boundaries \cite{4623242, 5754584, liu2022deep}. In practical imaging scenarios, however, images are often simultaneously affected by multiple degradation factors, such as noise, low contrast, and intensity inhomogeneity. These degradations obscure boundary information and significantly increase the difficulty of accurate segmentation. Therefore, developing robust segmentation methods for multi-degraded images remains an important and challenging problem\cite{liu2025convex, doi:10.1137/24M1676612}.

To address these challenges, existing image segmentation methods can be broadly divided into two categories: model-driven methods and data-driven methods. Model-driven methods include traditional segmentation methods, such as thresholding and watershed segmentation, and level set segmentation. Level set methods, proposed by Osher and Sethian \cite{OSHER198812} , have emerged as a powerful tool due to its ability to handle complex topological changes and high edge accuracy. In this framework, the segmentation problems are transformed into a variational formulation. Level set models can be broadly categorized into two classes: edge-based models such as the snake model \cite{1988Snakes} and geodesic active contour model \cite{466871}, and region-based models including the Mumford-Shah model \cite{1989Optimal} and Chan-Vese model \cite{902291}. To eliminate the need for re-initialization, Li et al. \cite{5557813} proposed the distance regularized level set evolution (DRLSE) model, which introduces a distance regularization term that gives the level set evolution with a unique forward-and-backward diffusion effect. Subsequently, Li et al. developed the region-scalable fitting (RSF) model \cite{4623242} and the local intensity clustering-based (LIC) fitting energy model \cite{5754584} to handle intensity inhomogeneity. Other researchers have incorporated various geometric priors, such as point distance shape \cite{LI2019123} and convex shape constraints \cite{9109685}, to achieve desired segmentation results. In response to the segmentation challenges posed by multiple degraded images, some scholars have attempted to construct integrated joint solution models, such as deeply coupling denoising modules or bias field correction modules with segmentation modules. For the images with intensity inhomogeneity, Ali et al. \cite{8341753} proposed the VLSGIS model ,which perform denoising during the segmentation process and correct the inhomogeneity. Zhou et al. \cite{zhou_seg} proposed a joint variation model to process the images with Gamma noise. Li et al. \cite{doi:10.1137/20M1310618} proposed a new three-stage segmentation framework based on a convex variant of the Mumford--Shah model and the intensity inhomogeneity information of an image.

Although existing variational level set methods have achieved remarkable success in many image segmentation tasks, they still suffer from several fundamental limitations. First, many models couple multiple energy terms within a single variational framework, making theoretical analysis of existence, uniqueness, and stability difficult. Second, most models rely on length regularization to maintain contour smoothness. However, the associated nonlinear evolution not only increases computational cost but also complicates numerical optimization. Finally, iterative level set evolution remains computationally expensive, particularly for large-scale images.


Recently, data-driven methods have achieved remarkable success across various visual tasks. Among these, convolutional neural networks (CNNs) have long served as the backbone for image segmentation. In particular, the symmetric encoder-decoder architecture U-Net and its variants \cite{ronneberger2015u, 10.1007/978-3-031-25066-8_9,CHEN2024103280} effectively integrate high-resolution contextual features with spatial details, enabling precise pixel-wise segmentation even under limited training data. To address the limitation caused by local receptive fields, Vision Transformers \cite{DBLP:journals/corr/abs-2010-11929} have emerged as a powerful alternative, offering superior global modeling capabilities through self-attention mechanisms. More recently, Segment Anything Model (SAM) \cite{2023Segment} has demonstrated impressive zero-shot segmentation capability. Although data-driven models have achieved remarkable performance on large-scale annoted datasets, their training cost and dependence on annotations remain important limitations.

To mitigate the interference of noise, texture, and grayscale variations on segmentation performance, image decomposition techniques have been widely incorporated into traditional segmentation workflows. The well-known Rudin-Osher-Fatemi (ROF) model \cite{RUDIN1992259} decomposes an image $f$ into a clean part $u$ and a noise part $n$ by minimizing an energy functional that balances fidelity and total variation regularization. Roy et al. \cite{doi:10.1137/24M167411X} proposed a three-component decomposition model that separates any grayscale image into its structural part, smooth part, and oscillatory part. Although image decomposition is always performed in a single stage, which may be mathematically ill-posed in some cases. The idea of separating structural information from degradation components provides an important motivation for the proposed framework.

Motivated by these observations, we propose a grayscale level set framework for image segmentation in this paper. The key idea is to transform a degraded image into a smooth image through a modular preprocessing stage and then restrict the admissible level set functions to those induced by gray-level thresholds. Consequently, the conventional continuous level set evolution is reformulated as a discrete threshold optimization problem, eliminating the need for explicit geometric regularization and substantially reducing the computational complexity.


Our main contributions in this work are the following:
\begin{enumerate}
	\item We introduce the concept of grayscale level set together with a new definition of smooth images, providing a reduced search space for regularization-free image segmentation.
	
	
	\item We establish the relationship between the grayscale level set space and the conventional level set space, and show that the optimal solution can be represented within the grayscale level set space under suitable smoothness assumptions.
	
	\item We develop an incremental update algorithm that reduces the computational complexity from
	from $\mathcal{O}(M \cdot N \cdot S)$ to $\mathcal{O}(N \cdot S)$.
	
\end{enumerate}

The remainder of this paper is organized as follows. Section 2 reviews the related work. Section 3 presents the proposed grayscale level set framework, including the underlying motivation, the proposed methodology, theoretical analysis, and complexity analysis. Section 4 validates the proposed method through extensive experiments and comparisons with representative approaches. Finally, Section 5 concludes the paper and discusses future research directions.

\section{Related Work}\label{sec2}

In this section, we review some relevant background knowledge that inspired our current work, including level set segmentation models and multi-task joint models.

\subsection{Region-Scalable Fitting Model}

The region-scalable fitting model(RSF) draws upon intensity information in local regions at a controllable scale. Compared with the global fitting energy proposed by CV model, RSF model defines the following local intensity fitting energy:
\begin{equation}
	\mathcal{E}_{\mathbf{x}}^{\text{Fit}}(\phi, f_{1}(\mathbf{x}), f_{2}(\mathbf{x})) = \sum_{i=1}^{2} \lambda_{i} \int_{\Omega} K_{\sigma}(\mathbf{x}-\mathbf{y}) | I(\mathbf{y}) - f_{i}(\mathbf{x}) |^{2}M_i(\phi(\mathbf{y})) \, d\mathbf{y}
\end{equation}
where $\lambda_{1}$ and $\lambda_2$ are positive constants, and $f_1(x)$ and $f_2(x)$ are two values that approximate image intensities in $\Omega_1$ and $\Omega_2$, respectively. Then, the model minimizes the energy $\mathcal{E}_{\mathbf{x}}^{\text{Fit}}$ for all $x$ in the image domain $\Omega$ by minimizing the integration of all central points in the image domain $\Omega$. The whole level set energy functional is as follows:
\begin{equation}
	\mathcal{E}\bigl(\phi, f_{1}(\mathbf{x}), f_{2}(\mathbf{x})\bigr) 
	= \int_{\Omega}\mathcal{E}_{\mathbf{x}}^{\text{Fit}}(C, f_{1}(\mathbf{x}), f_{2}(\mathbf{x}))d\mathbf{x}+\nu |C|
\end{equation}
The corresponding gradient descent flow is:
\begin{equation}
	\frac{\partial \phi}{\partial t} =\nu \delta_{\epsilon}(\phi) \operatorname{div}\left(\frac{\nabla \phi}{|\nabla \phi|}\right) -\delta_{\epsilon}(\phi)\left(\lambda_{1} e_{1} - \lambda_{2} e_{2}\right)
\end{equation}
where $\delta_{\epsilon}$ is the smoothed Dirac delta function and $e_1$ and $e_2$ are the functions $e_i(x)=\int K_{\sigma}(y-x)|I(x)-f_i(x)|^2d\mathbf{y}, i=1,2$.

\subsection{Multi-task Joint Models}

Several researchers have proposed joint models that simultaneously perform denoising and segmentation within a variational framework. Haider Ali et al. \cite{8341753} proposed a variational level set segmentation framework suitable for the  images which have intensity inhomogeneity, noise and a combination of both. 

\begin{equation}
	F^{TP}(\phi) = \mu \int_{\Omega}|\nabla H_{\epsilon}(\phi(\mathbf{x}))|d\mathbf{x} 
	+ (1-\lambda)\int_{\Omega}(z(\mathbf{x})-z_0(\mathbf{x}))^2d\mathbf{x} 
	+ \lambda \int_{\Omega}(\log z(\mathbf{x})+\frac{z_0(\mathbf{x})}{z(\mathbf{x})})d\mathbf{x}
\end{equation}
where $\mu$ and $\lambda$ are given weight constants and $\int_{\Omega}|H_{\epsilon}(\phi(\mathbf{x}))|d\mathbf{x}$ is length regularization term. This formulation enables the model to distinguish meaningful object boundaries from noise-induced artifacts. The corresponding gradient descent flow is as follows:
\begin{equation}
	\frac{\partial\phi}{\partial t} = \delta_{\epsilon}(\phi(\mathbf{x}))\mu\operatorname{div}\left(\frac{\nabla \phi}{|\nabla \phi|}\right)-D\delta_{\epsilon}(\phi(\mathbf{x}))\cdot		(2(1-\lambda)(z(\mathbf{x})-z_0(\mathbf{x})+\lambda\frac{z(\mathbf{x})-z_0(\mathbf{x})}{z(\mathbf{x})^2})
\end{equation}
with Neumann boundary conditions, and $D = (f_1(\mathbf{x})-f_2(\mathbf{x}))$. Experiments demonstrate that the proposed model achieves superior segmentation accuracy and robustness in the presence of severe intensity inhomogeneity and mixed noise conditions.
%

Similarly, Hsieh \cite{doi:10.1137/24M1676612} proposed a novel variational model based on an additive bias correction for image segmentation in the presence of noise and intensity inhomogeneity. The segmentation model includes TV denoising, local image smoothing energies, local bias-correction segmentation, and length regularization, formulated as
\begin{equation}
	\min_{(B, c, S, \chi) \in L^{2}(\Omega) \times \mathbb{R}^{m} \times BV(\Omega) \times X} \mathcal{E}(B, c, S, \chi)
\end{equation}
where the total energy functional $\mathcal{E}$ is given by
\begin{equation}
	\begin{aligned}
		&\mathcal{E}(B, c, S, \chi) = \frac{\nu}{2} \iint K_{s}(x-y)(I(y)-S(x))^{2} \, dy \, dx\\
		& + \sum_{i=1}^{m} \frac{\lambda_{i}}{2}\iint K_{\sigma}(x-y)(S(y)-B(x)-c_{i})^{2} \chi_{i}(y) \, dy \, dx \\
		&+ \mu \sum_{i=1}^{m} \sum_{\substack{j=1 \\ j \neq i}}^{m} \sqrt{\frac{\pi}{\tau}} \int_{\Omega} \chi_{i}(x)\left(G_{\tau} * \chi_{j}\right)(x) \, dx\\
		&+\alpha \int_{\Omega} |\nabla S(x)| \, dx\\
	\end{aligned}
\end{equation}

In addition, authors develop an ICT method to efficiently implement the proposed model. This model is robust for segmentation noisy images with intensity inhomogeneity correction, and noise removal.

\section{Grayscale Level Set Framework}\label{sec3}

\subsection{Motivation: Limitations of Length Regularization}\label{subsec31}

In traditional level set segmentation frameworks, regularization terms are essential components for ensuring stable evolution. Their core role is to constrain the morphology and evolution behavior of the level set function, compensating for deficiencies in the fitting term. Common regularization terms include length regularization, area regularization, and curvature regularization. The length regularization term has become one of the most widely used constraints because it effectively controls jagged distortions of the segmentation boundary, avoids isolated points, and maintains boundary smoothness. The length term typically takes the form:
\begin{equation}
	\mathcal{L}(\phi) = \int_{\Omega} \left| \nabla H(\phi) \right| \, d\mathbf{x}
\end{equation}
with corresponding gradient flow:
\begin{equation}
	\frac{\partial \phi}{\partial t} = \delta(\phi) \operatorname{div} \left( \frac{\nabla \phi}{|\nabla \phi|} \right)
\end{equation}
The gradient flow corresponding to the length regularization term can be regarded as the total variation flow at the zero level set. The total variation flow with Neumann boundary conditions is formalized as:
\begin{equation}
	\begin{cases}
		\dfrac{\partial u}{\partial t} = \operatorname{div} \left( \dfrac{\nabla u}{|\nabla u|} \right), & \text{in } Q = (0,\infty) \times \Omega, \\[8pt]
		\dfrac{\partial u}{\partial \eta} = 0, & \text{on } S = (0,\infty) \times \partial\Omega, \\[8pt]
		u(0, x) = u_0(x), & \text{in } x \in \Omega,
	\end{cases}
	\label{total variation flow}
\end{equation}
where $\Omega$ is a bounded set in $\mathbb{R}^{N}$ with Lipschitz continuous boundaty $\partial \Omega$ and $u_0\in L^1(\Omega)$.

\begin{lemma}
	Suppose $N=2$. Let $u_0\in L^2(\Omega)$ and $u(t,x)$ the unique weak solution of problem(\ref{total variation flow}). Then there exists a finite time $T_0$ such that
	\begin{equation}
		u(t) = (u_0)_{\Omega} = \frac{1}{\mathcal{L}^{N}(\Omega)} \int_{\Omega} u_0(x) \, dx \qquad \forall \, t \geq T_0
	\end{equation}
	\label{total variation flow theorem}
\end{lemma}

\begin{figure}[htbp]
	\centering
	\begin{subfigure}[b]{0.25\textwidth}
		\centering
		\includegraphics[width=\linewidth]{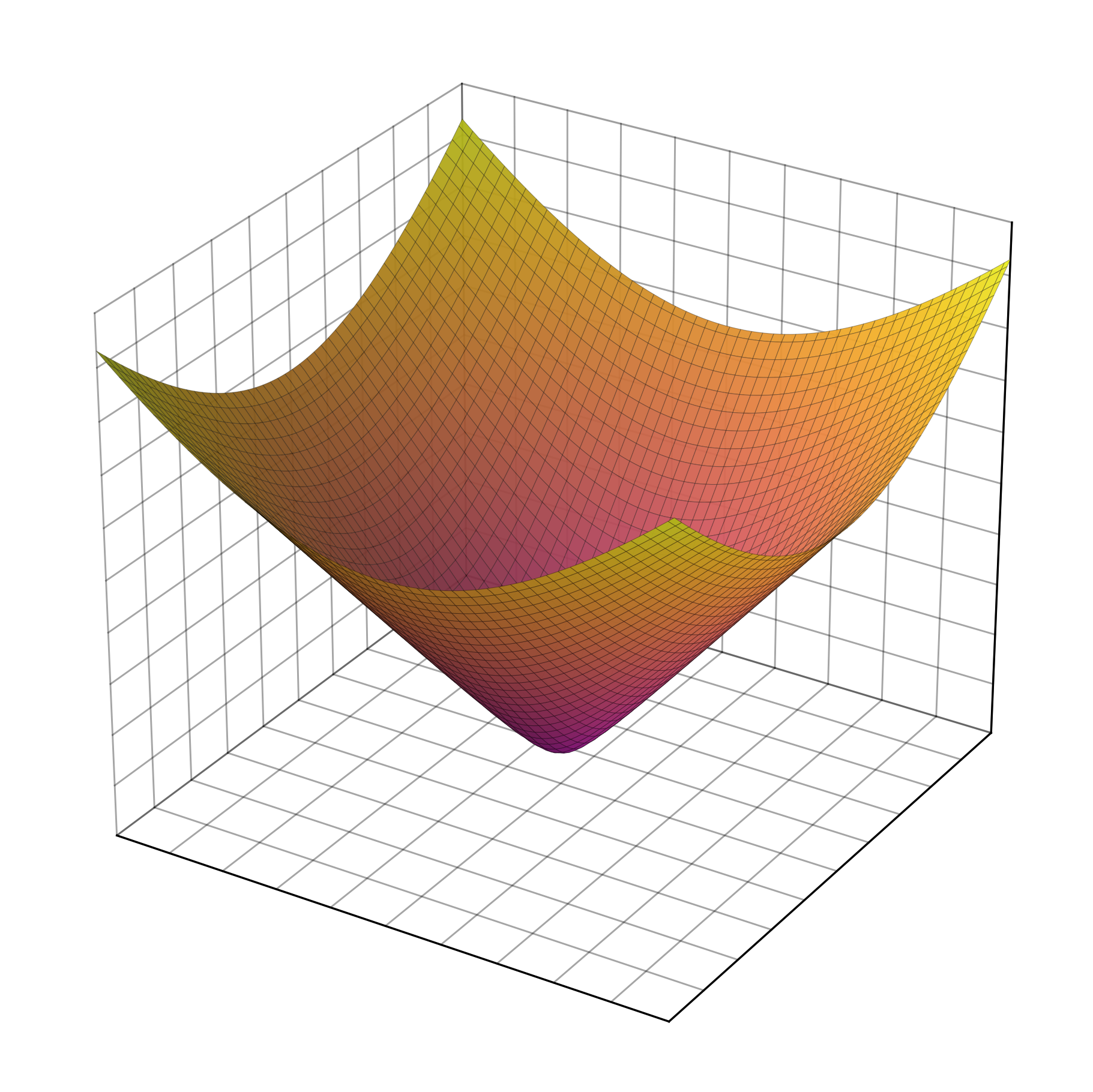}
		\caption{100 iters}
		\label{fig:100b}
	\end{subfigure}
	\begin{subfigure}[b]{0.25\textwidth}
		\centering
		\includegraphics[width=\linewidth]{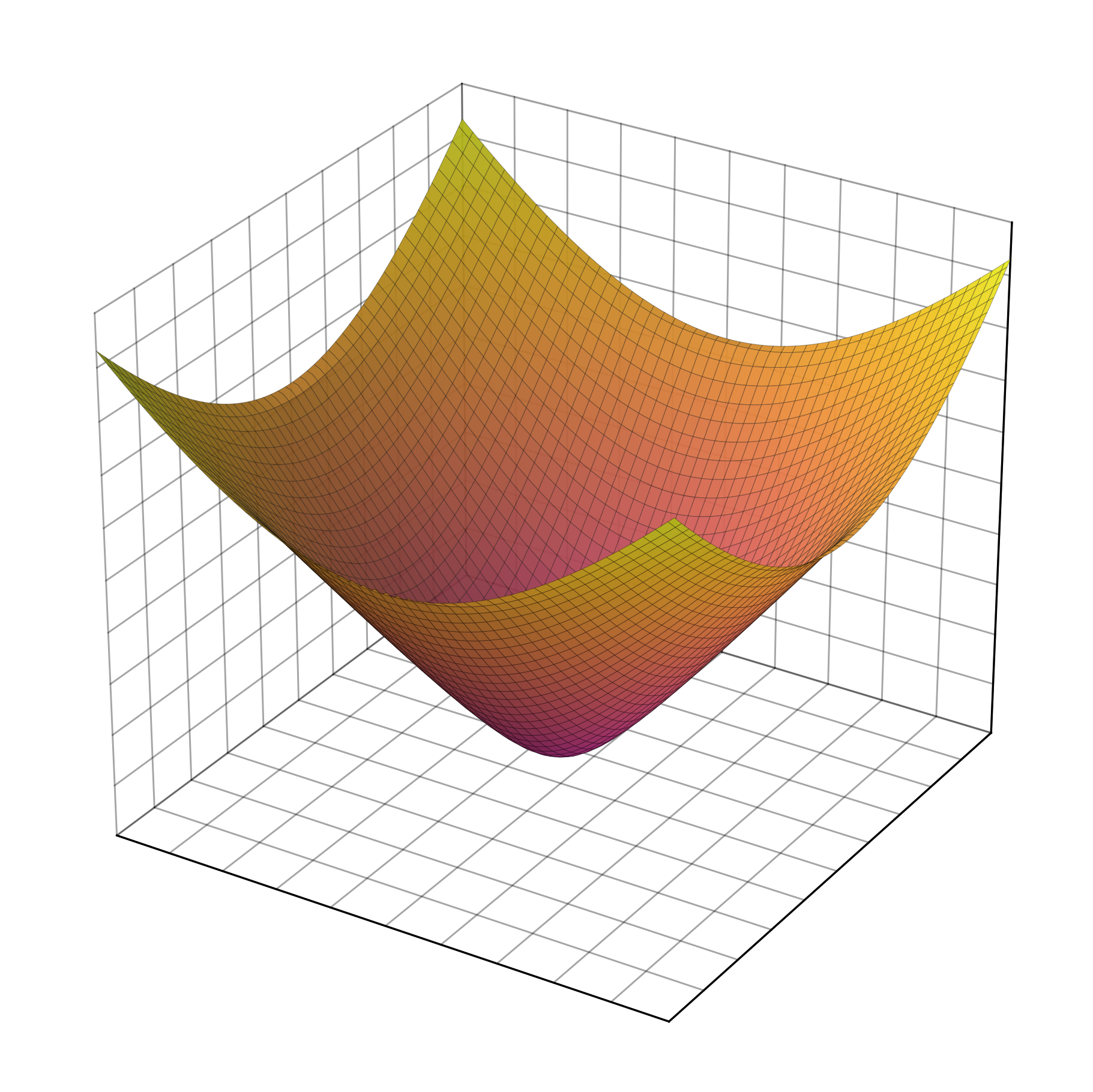}
		\caption{500 iters}
		\label{fig:500b}
	\end{subfigure}
	\begin{subfigure}[b]{0.25\textwidth}
		\centering
		\includegraphics[width=\linewidth]{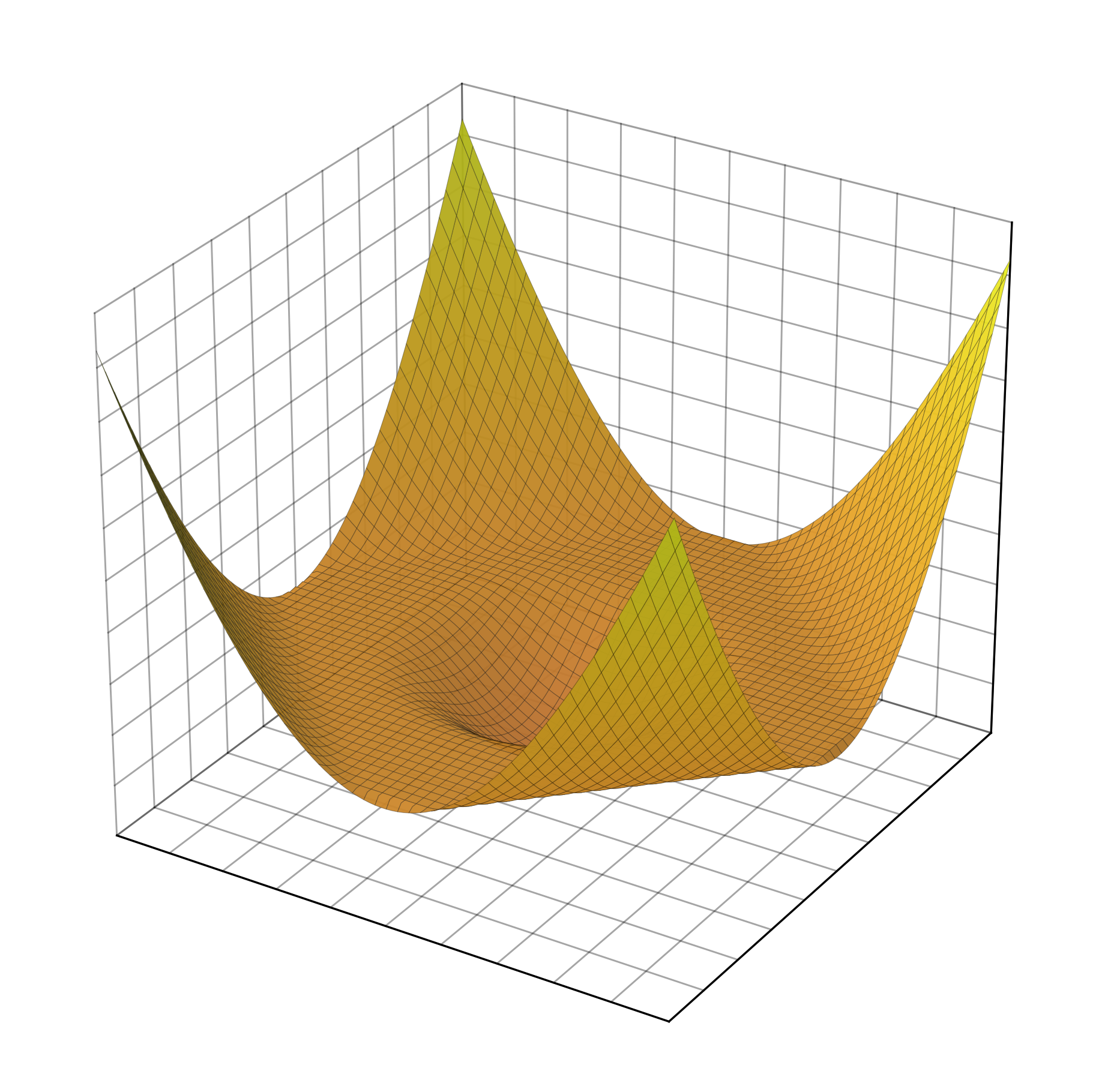}
		\caption{10000 iters}
		\label{fig:10000b_iter}
	\end{subfigure}
	\caption{Evolution of level set function corresponding to length regularization term with different iterations.}
	\label{fig:onlylength_evolution}
\end{figure}
\noindent For the total variation flow, we know that the solutions stabilize as $t\rightarrow 0$ from the lemma \ref{total variation flow theorem} and Fig. \ref{fig:onlylength_evolution}. The result indicates that curvature-driven smoothing tends to reduce gradient variations. 

Besides, traditional level set method usually adapt the distance regularization term to maintain the property of signed distance function avoiding frequent re-initialization. The coresponding gradient flow of distance regularization is as follows:
\begin{equation}
	\frac{\partial \phi}{\partial t} = \operatorname{div}(d(|\nabla\phi|\nabla\phi))
\end{equation}
While length term and distance regularization term are often combined to leverage their respective advantages, we find that there exists incompatibility from the evolutionary dynamics perspective: the length term constrains $|\nabla\phi|$ to approach $0$ near the zero level set region, while the distance regularization term enforces $|\nabla\phi|$ to approach $1$. The incompatibility is not merely a matter of computational overhead but represents a conflict that undermines the efficiency and stability of traditional methods.

\subsection{The Performance of Regularization-free}
From this perspective, we consider whether accurate and robust segmentation can be achieved by discarding these conflicting regularization terms entirely, only relying on a well-designed data fidelity term. To investigate this, we conduct a validation experiment on a real image severely corrupted by strong noise. Under such degradation, edge information is largely destroyed. As shown in Fig. \ref{3.1experimental} (d)-(f), traditional level set methods (CV, RSF, LIC) equipped with full regularization terms can only produce marginally acceptable results with poor performance on edge due to limited denoising capability. We then apply classical denoising to obtain a sufficiently smooth image and modify the three models by removing all regularization components (length and distance regularizers), retaining only the data-fitting energy term. Remarkably, as shown in Fig. \ref{3.1experimental} (g)-(i), all three regularization-free models yield clear, complete, and smoothly bounded segmentation results on the preprocessed smooth image. The segmentation quality is notably better than that of the original regularized models: no boundary distortion and isolated noise points.
\begin{figure}[htbp]
	\centering
	\begin{subfigure}{0.14\textwidth}
		\centering
		\includegraphics[scale=0.65]{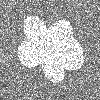}
		\caption{Degraded}
	\end{subfigure}
	\begin{subfigure}{0.14\textwidth}
		\centering
		\includegraphics[scale=0.65]{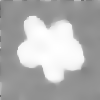}
		\caption{Smooth}
	\end{subfigure}
	\begin{subfigure}{0.14\textwidth}
		\centering
		\includegraphics[scale=0.65]{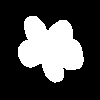}
		\caption{GT}
	\end{subfigure}
	\begin{subfigure}{0.14\textwidth}
		\centering
		\includegraphics[scale=0.65]{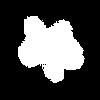}
		\caption{CV Result}
	\end{subfigure}
	\begin{subfigure}{0.14\textwidth}
		\centering
		\includegraphics[scale=0.65]{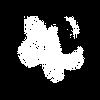}
		\caption{RSF Result}
	\end{subfigure}
	\begin{subfigure}{0.14\textwidth}
		\centering
		\includegraphics[scale=0.65]{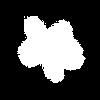}
		\caption{LIC Result}
	\end{subfigure}
	\begin{subfigure}{0.14\textwidth}
		\centering
		\includegraphics[scale=0.65]{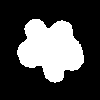}
		\caption{CV-Free}
	\end{subfigure}
	\begin{subfigure}{0.14\textwidth}
		\centering
		\includegraphics[scale=0.65]{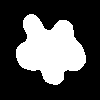}
		\caption{RSF-Free}
	\end{subfigure}
	\begin{subfigure}{0.14\textwidth}
		\centering
		\includegraphics[scale=0.65]{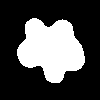}
		\caption{LIC-Free}
	\end{subfigure}
	\begin{subfigure}{0.14\textwidth}
		\centering
		\includegraphics[scale=0.85]{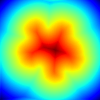}
		\caption{CV LSF}
	\end{subfigure}
	\begin{subfigure}{0.14\textwidth}
		\centering
		\includegraphics[scale=0.85]{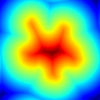}
		\caption{RSF LSF}
	\end{subfigure}
	\begin{subfigure}{0.14\textwidth}
		\centering
		\includegraphics[scale=0.85]{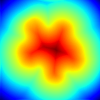}
		\caption{LIC LSF}
	\end{subfigure}
	\caption{The results of variation segmentation model.}
	\label{3.1experimental}
\end{figure}

More importantly, the level set functions obtained from these regularization-free models (Fig. \ref{3.1experimental} (j)-(l)) exhibit desirable properties, which morphology corresponds to the ideal subspace of level set functions for segmentation tasks. The underlying reason is that on a sufficiently smooth image, the grayscale distributions of the target and background regions become well-separable and their intensity ranges no longer overlap in a way that would mislead the fitting energy. Thus the fitting term alone provides sufficient discriminative power to distinguish the two regions.

Motivated by this finding, we discard the regularization design of traditional level sets and adopt a two-stage strategy: first preprocess the degraded image into a sufficiently smooth representation, then perform segmentation solely by optimizing the refined data-fitting energy.

\subsection{Grayscale Level Set Framework}\label{subsec32}
The observations in Section~\ref{subsec31} motivate the proposed grayscale level set framework. In this subsection, we introduce the proposed methodology, describe its efficient implementation, and present the corresponding complexity and theoretical analyses.

\subsubsection{Preceding Preprocessing}
For images with multiple degradations, the traditional level set method usually couples image denoising, bias field correction and boundary geometric regularization in the same complex variational functional. This highly coupled model will face difficulties in numerical calculation. In this framework, preceding preprocessing constrains images to prepare for regularization-free segmentation.

For digital images in real scenarios, the imaging process can be mathematically described by several general models with clear physical interpretation, covering the widely existing additive noise and multiplicative noise scenarios in computer vision tasks \cite{doi:10.1137/24M1676612}\cite{5754584}.

\textbf{Additive model} for images from incoherent imaging system with additive noise and additive bias field:
\begin{equation}
	I(x) = B(x) + J(x) + N(x)
\end{equation}

\textbf{Multiplicative model} for images with multiplicative speckle noise, such as synthetic aperture radar (SAR) images:
\begin{equation}
	I(x) = B(x) * J(x) * N(x)
\end{equation}

\textbf{Mixture model} for images with additive noise and intensity inhomogeneity multiplicative bias field, commonly appearing in medical imaging, such as MRI or CT:
\begin{equation}
	I(x) = B(x) * J(x) + N(x)
	\label{Mixture model}
\end{equation}
where $x\in\Omega$ denotes the pixel coordinate in the 2D bounded open image domain $\Omega\subset\mathbb{R}^2$, and the three components have clear physical meanings:
\begin{itemize}
	\item $J(x)$: The structural component, which carries the core semantic information of the image.
	\item $B(x)$: The bias component, describing the low frequency slowly varying intensity distribution caused by imaging system bias field. 
	\item $N(x)$: The oscillatory component, containing high frequency interferences such as imaging noise.
\end{itemize}

\begin{figure}[htbp]
	\centering
	\begin{subfigure}{0.20\textwidth}
		\centering
		\includegraphics[scale=1.2]{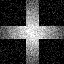}
		\caption{$I(x)$}
	\end{subfigure}
	\begin{subfigure}{0.20\textwidth}
		\centering
		\includegraphics[scale=1.2]{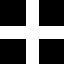}
		\caption{$J(x)$}
	\end{subfigure}
	\begin{subfigure}{0.20\textwidth}
		\centering
		\includegraphics[scale=1.2]{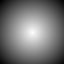}
		\caption{$B(x)$}
	\end{subfigure}
	\begin{subfigure}{0.20\textwidth}
		\centering
		\includegraphics[scale=1.2]{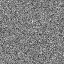}
		\caption{$N(x)$}
	\end{subfigure}
	\caption{Visualization of Mixture Model (\ref{Mixture model})}
\end{figure}
The first step of our framework is to construct a generalized preprocessing mapping operator $S_\sigma: BV(\Omega) \rightarrow H^1(\Omega)$ parameterized by a scale $\sigma$, which maps the observed image $I(x)$ to a smooth image $J(x) $ with good properties.
\begin{equation}
	J(x) = S_\sigma(I(x))
\end{equation}
Preprocessing mapping operator $S_{\sigma}$ has a high degree of modularity and scalability. Common choices include total variation denoising, Gaussian filtering, bias field correction, dictionary learning, or deep learning-based enhancement modules.

Below, we define a smooth image to measure whether the selected preprocessing operator is suitable for subsequent segmentation
\begin{definition}
	Let $\Omega\subset \mathbb{R}^2$ be a Lipschitz region. The image $J:\Omega\rightarrow\mathbb{R}$ is called $H^1$ \textbf{smooth image} if it satisfies the following conditions:
	\begin{itemize}
		\item $J\in H^1(\Omega)$, and there exists a constant $L>0$ such that: 
		\begin{equation}
			|\nabla J(x)|\leq L \quad a.e. in\ \Omega.
		\end{equation}
		\item For target partition $(\Omega_1, \Omega_2)$, there exists a threshold $c^*$ such that:
		\begin{equation*}
			\sup_{\Omega_{2}} J < c^* < \inf_{\Omega_{1}}J
		\end{equation*}
	\end{itemize}
	\label{definition_smooth_image}
\end{definition}
\noindent The first point of the definition controls image noise, and the second point provides guarantee for subsequent segmentation.

\subsubsection{Grayscale Level Set Segmentation Method}
Under the condition of smooth images, we can directly use image gray information to induce a subset of the traditional level set, which is called \textbf{grayscale level set} here.
\begin{definition}
	Let $J:\Omega\rightarrow\mathbf{R}$ be a smooth image satisfying definition \ref{definition_smooth_image}. For any gray-level threshold $c\in[C_{min},C_{max}]$, we define the two regions:
	\begin{equation*}
		\Omega_{1}=\{x:J(x)>c\}, \Omega_{2}=\Omega\backslash\Omega_{1},
	\end{equation*}
	denote the interface:
	\begin{equation*}
		\Gamma_c=\partial\Omega_{1}(c),
	\end{equation*}
	and the \textbf{grayscale level set} can be defined as: 
	\begin{equation}
		\Phi_{c}(x) = \{\phi_c(x) = SDF(\Gamma_c) | c\in[C_{min},C_{max}]\}
	\end{equation}
	where $SDF(\Gamma_c)$ is the signed distance function defined on $\Gamma_c$.
\end{definition}

In fact, the grayscale level set is the contour lines of a smooth continuous grayscale surface $J (x) $truncated at different heights. Fig: \ref{geometric interpretation} shows the zero level contour lines corresponding to different grayscale level sets which naturally have good geometric properties.
\begin{figure}[htbp]
	\centering
	\begin{subfigure}{0.18\textwidth}
		\centering
		\includegraphics[width=0.95\textwidth]{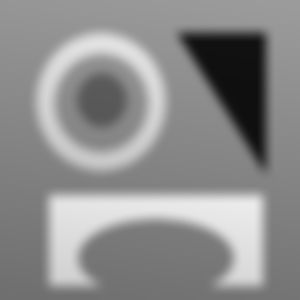}
		\caption{$J(x)$}
	\end{subfigure}
	\hspace{30pt}
	\begin{subfigure}{0.18\textwidth}
		\centering
		\includegraphics[width=0.95\textwidth]{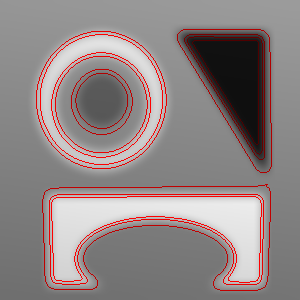}
		\caption{Contour}
	\end{subfigure}
	\caption{The Gray Contour of Smooth Images}
	\label{geometric interpretation}
\end{figure}

\begin{proposition}
	The grayscale level set is a subspace of the traditional level set, i.e
	\begin{equation*}
		\Phi_c(x)\subset\Phi(x)
	\end{equation*}
	\label{proposition1}
\end{proposition}
\begin{proof}
	For any threshold $c$, we have $$\phi_c(x)=SDF(\Gamma _c)\in\Phi.$$ 
	Hence, $\Phi_c\subset\Phi$.
\end{proof}
\noindent Besides, the grayscale level set has the above relationship with the traditional level set. The traditional level set method updates the level set function through variational methods, which can easily get stuck in local minima and be computationally complex. The grayscale level set method compresses the entire search domain to the grayscale level of the image itself, which has significant computational advantages.

We obtains the final segmentation result by minimizing the following energy. The optimal threshold $c^*$ is obtained by minimizing the regularization-free local region scalable fitting energy:

\begin{equation}
	c^{*} = \arg\min_{c \in [C_{min},C_{max}]} \sum_{i=1}^{2} \lambda_i\int_{\Omega} \int_{\Omega_i(c)} K_{\sigma}(x - y) \,\bigl|J(y) - f_i(x; c)\bigr|^2 \,dy \,dx,
	\label{RSF}
\end{equation}
where $\lambda_{1}, \lambda_{2}>0$ are weight coefficients, $K_{\sigma}$ is a truncated Gaussian kernel with standard deviation $\sigma$ for controlling the local fitting scale, and the two regions are defined as $\Omega_{1}(c)=\{x\in\Omega|J(x)>c\}$ and $\Omega_{2}(c)=\Omega\backslash\Omega_{1}(c)$. The expressions of the local intensity fitting functions $f_1(x;c)$ and $f_2(x;c)$ are:
\begin{equation}
	f_i(x; c) = \frac{\int_{\Omega_i(c)} K_{\sigma}(y - x) J(y) \, dy}{\int_{\Omega_i(c)} K_{\sigma}(y - x) \, dy}, \quad i = 1, 2.
	\label{local fitting term}
\end{equation}
The final segmentation boundary is defined as the zero level set of the optimal gray level set function:
\begin{equation}
	\Gamma^{*} = \partial\bigl\{x \in \Omega : J(x) > c^{*}\bigr\}
\end{equation}

In the actual calculation process, we found that the local fitting function $f_i(x,c)$ and energy $\mathcal{E}$ have to be recalculated globally for each threshold $c$, which will lead to a large number of repeated calculations. Inspired by the narrowband method \cite{ADALSTEINSSON1995269}, we propose the incremental update implementation. When $c$ increases, the region division changes only in the gray level interval $[c, c+\Delta c]$, and the change of corresponding $f_i(x,c)$ and $\mathcal{E}$ only occurs in the neighborhood of these gray levels, so we only need to update in the area when calculating the local intensity function and energy. The complete incremental update implementation is summarized in Algorithm \ref{alg:incremental update}.

\begin{algorithm}[h]
	\caption{Incremental update implementation}
	\label{alg:incremental update}
	\begin{algorithmic}
		\Require Normalized smooth image $J(x)$, threshold step $\Delta c$, Kernel $K_\sigma$
		\Ensure Optimal segmentation boundary $\Gamma^*$
		
		\Statex
		\State \textbf{Step 1: Initialization}
		\State 1. Determine the range of grayscale values in $J$: $[C_{min}, C_{max}]$.
		\State 2. Initialize $c = C_{min}$, $c^* = C_{min}$, $E_{min} = \infty$.
		\State 3. Initialize regions: $\Omega_1(C_{min}) = \emptyset$, $\Omega_2(C_{min}) = \Omega$ and compute initial global energy $E(C_{min})$.
		\Statex
		\State \textbf{Step 2: Incremental Threshold Search}
		\For{$c = C_{min}$ to $C_{max}$ with step $\Delta c$}
		\State 1. Identify pixel set: $\Delta \Omega(c) = \{x \in \Omega \mid c \le J(x) < c + \Delta c\}$.
		\If{$\Delta \Omega(c)$ is empty}
		\State \textbf{continue} to next threshold.
		\EndIf
		\State 2. Update region partitions.
		\State 3. Incremental update the local fitting function $f_i(x,c)$ for pixels $x$ within the local neighborhood $B(y, \sigma)$ of any $y \in \Delta \Omega(c)$.
		\State 4. Update energy variation $\Delta E$ based on the updated $f_1, f_2$ and compute $E(c)$.
		\If{$E(c) < E_{min}$}
		\State $E_{min} \leftarrow E(c)$ $\quad  c^* \leftarrow c$.
		\EndIf
		\EndFor\\
		\Return segmentation boundary $\Gamma^* = \partial\Omega_1(c^*)$.
	\end{algorithmic}
\end{algorithm}

\subsubsection{Complexity Analysis}
Let $N$ be the total number of pixels in the image $\Omega$, $S$ be the number of pixels within the compact support domain of the kernel $K_\sigma$ (where $S \approx \pi \sigma^2$), and $M$ be the total number of discrete threshold steps ($M = \frac{C_{max} - C_{min}}{\Delta c}$).In the direct implementation, for each threshold $c$, the local fitting functions $f_i(x, c)$ and the energy functional $E(c)$ require integrating over the entire image domain. The convolution operation at each step demands $O(N \cdot S)$ operations. Over $M$ steps, the total computational complexity is:
\begin{equation*}
	\mathcal{O} = \mathcal{O}(M \cdot N \cdot S)
\end{equation*}

The incremental update algorithm (Algorithm~\ref{alg:incremental update}) is based on the observation that only a sparse subset of pixels, denoted by $\Delta\Omega(c)$, changes its region assignment at each threshold update. The key observation is that throughout the entire search process from $C_{min}$ to $C_{max}$, each pixel in the image transitions from $\Omega_2$ to $\Omega_1$ exactly once.
Therefore, the sum of the cardinalities of all changing sets across all $M$ steps is exactly equal to the total number of pixels $N$:
\begin{equation*}
	\sum_{k=1}^{M} |\Delta \Omega(c_k)| = N
\end{equation*}
The total computational effort over the entire optimization process transforms from repeated global calculations to a single global initialization plus a series of sparse local updates:
\begin{equation*}
	\begin{aligned}
		\mathcal{O}_{Algo1} &= \mathcal{O}(N \cdot S) + \sum_{k=1}^{M} \mathcal{O}(|\Delta \Omega(c_k)| \cdot S) \\
		&= \mathcal{O}(N \cdot S) + \mathcal{O}(N \cdot S) \\
		&= \mathcal{O}(N \cdot S)
	\end{aligned}
\end{equation*}
Algorithm \ref{alg:incremental update} demonstrates that our method effectively decouples the computational complexity from the number of discrete thresholds $M$, which significantly improves the computational efficiency

\subsubsection{Theoretical Analysis}
In this part, we theoretically analyze the existence and global optimality of gray level set solution.
\begin{theorem}
	For any smooth image \( J \in H^1(\Omega) \) satisfying the constraints of Definition \ref{definition_smooth_image}, the induced grayscale level set segmentation curves \( \Gamma_c \) exist almost everywhere, and their one-dimensional Hausdorff measure (i.e., curve perimeter) is bounded, satisfying
	\[
	\int_{C_{\min}}^{C_{\max}} \mathcal{H}^1(\Gamma_c) \, dc \leq L \cdot |\Omega| < \infty,
	\]
	where \( |\Omega| \) denotes the Lebesgue measure of the domain \( \Omega \).
	\label{thm:existence_finite_length}
\end{theorem}
\begin{proof}
	By the co-area formula, for a Sobolev function \( J \in H^1(\Omega) \) defined on a bounded Lipschitz domain \( \Omega \), the integral of the total variation is related to the edge lengths of essential level sets by the following identity:
	\[
	\int_{-\infty}^{\infty} \mathcal{H}^1(\{x \in \Omega \mid J(x) = c\} \cap \Omega) \, dc = \int_{\Omega} |\nabla J(x)| \, dx.
	\]
	Since \( J(\Omega) \subset [C_{\min}, C_{\max}] \), the level sets for \( c \notin [C_{\min}, C_{\max}] \) are empty (with zero measure), so the integration interval on the left-hand side can be reduced to \( [C_{\min}, C_{\max}] \). Moreover, by condition 1 in Definition \ref{definition_smooth_image}, the weak gradient satisfies the global uniform bound \( |\nabla J(x)| \leq L \) a.e. Therefore, the right-hand side can be estimated as
	\[
	\int_{\Omega} |\nabla J(x)| \, dx \leq \int_{\Omega} L \, dx = L \cdot |\Omega|.
	\]
	Because the domain is bounded (\( |\Omega| < \infty \)), we obtain
	\[
	\int_{C_{\min}}^{C_{\max}} \mathcal{H}^1(\Gamma_c) \, dc \leq L \cdot |\Omega| < \infty.
	\]
\end{proof}

\begin{theorem}[Optimality in the grayscale-level-set space]
	Let $J:\Omega\to [C_{\min},C_{\max}]$ be a smooth image satisfying the
	separation condition: there exist a partition $(\Omega_1^*,\Omega_2^*)$
	and a threshold $c^*\in[C_{\min},C_{\max}]$ such that
	\[
	\Omega_1^*=\{x\in\Omega:J(x)>c^*\},\qquad
	\Omega_2^*=\Omega\setminus\Omega_1^* .
	\]
	Assume that $(\Omega_1^*,\Omega_2^*)$ is a global minimizer of the
	regularization-free fitting energy over the admissible partition class
	$\mathcal A$. Then
	\[
	\min_{c\in[C_{\min},C_{\max}]}\mathcal E_{\rm gray}(c)
	=
	\min_{(\Omega_1,\Omega_2)\in\mathcal A}\mathcal E(\Omega_1,\Omega_2).
	\]
\end{theorem}

\begin{proof}
	Since each threshold $c$ induces an admissible partition
	\[
	\Omega_1(c)=\{x\in\Omega:J(x)>c\},\qquad
	\Omega_2(c)=\Omega\setminus\Omega_1(c),
	\]
	the grayscale-level-set partitions form a subset of the admissible
	partition class $\mathcal A$. Hence,
	\[
	\min_{c\in[C_{\min},C_{\max}]}\mathcal E_{\rm gray}(c)
	\geq
	\min_{(\Omega_1,\Omega_2)\in\mathcal A}\mathcal E(\Omega_1,\Omega_2).
	\]
	
	On the other hand, by the separation assumption, the global minimizer
	$(\Omega_1^*,\Omega_2^*)$ can be represented by the threshold $c^*$, i.e.,
	\[
	(\Omega_1^*,\Omega_2^*)=(\Omega_1(c^*),\Omega_2(c^*)).
	\]
	Therefore,
	\begin{equation*}
		\min_{c\in[C_{\min},C_{\max}]}\mathcal E_{\rm gray}(c)
		\leq
		\mathcal E_{\rm gray}(c^*)
		=
		\mathcal E(\Omega_1^*,\Omega_2^*)
		=
		\min_{(\Omega_1,\Omega_2)\in\mathcal A}\mathcal E(\Omega_1,\Omega_2).
	\end{equation*}
	Combining the two inequalities yields the desired equality.
\end{proof}

\section{Experimental Results}\label{sec4}
In this section, we conduct comprehensive experiments to evaluate the effectiveness, computational efficiency, and robustness of the proposed grayscale level set segmentation framework. The proposed method is assessed on both synthetic and real-world images under various degradation conditions, and is compared with several representative level set models using quantitative metrics and qualitative visual results.

\subsection{Evaluation Metrics}
To quantitatively evaluate the proposed framework, five metrics are adopted. Dice, IoU, and HD95 are used to assess segmentation accuracy, while the boundary length and mean absolute curvature (MAC) are employed to evaluate the geometric quality of the resulting contours.

\begin{itemize}
	\item \textbf{Dice}: The Dice is employed to quantify the spatial overlap accuracy between the segmented region and the ground truth.
	\begin{equation*}
		\text{Dice}(A, B) = \frac{2 |A \cap B|}{|A| + |B|}
	\end{equation*}
	
	\item \textbf{IoU}: The IoU evaluates region similarity by calculating the ratio of the intersection to the union of the predicted and ground-truth domains.
	\begin{equation*}
		\text{IoU}(A, B) = \frac{|A \cap B|}{|A \cup B|}
	\end{equation*}
	
	\item \textbf{HD95}: The HD95 measures the 95th percentile of the bidirectional directed distances between the segmented boundary $\partial A$ and the ground-truth boundary $\partial B$, which effectively eliminates the adverse impact of outlier noises.
	\begin{equation*}
		P_{95} \left( \max_{a \in \partial A} \min_{b \in \partial B} \|a - b\|, \max_{b \in \partial B} \min_{a \in \partial A} \|b - a\| \right)
	\end{equation*}
	\item \textbf{Length}: The Length metric characterizes the geometric perimeter of the segmentation boundary via discrete curvilinear integration. It is utilized to quantitatively appraise boundary regularity.
	\begin{equation*}
		\text{Length}(\partial A) = \int_{\partial A} \text{d}s
	\end{equation*}
	
	\item \textbf{MAC}: The MAC reflects geometric smoothness of the segmented boundary. A lower MAC value indicates a smoother boundary.
	\begin{equation*}
		\text{MAC}(\partial A) = \frac{1}{\mathcal{L}(\partial A)} \int_{\partial A} |\kappa(s)| \text{d}s
	\end{equation*}
\end{itemize}

\subsection{Validation on Synthetic Images with Intensity Inhomogeneity and Low Contrast}

To verify the segmentation performance of the proposed algorithm and clarify the algorithm's adaptability to image degradation factors, we design three typical complex scenarios: intensity inhomogeneity, noise, and low contrast. For image 1, we add additive Gaussian noise ($\mu=0$, $\sigma=0.2$) and a gradient bias field from the top left corner to the bottom right corner. Image 2 simulates the low contrast case with Gaussian noise($\mu=0$, $\sigma=0.1$) and image 3 integrates various degradation of intensity inhomogeneity, low contrast and noise.

According to the imaging model described in Section 3.2.1, each synthetic image is first processed by BM3D \cite{4271520}, PDE filtering \cite{WEN2021108049}, N4ITK \cite{5445030} etc. After obtaining the smooth image, we set $\lambda_{1}=\lambda_{2}=1$ for image 1 and $\lambda_{1}=1, \lambda_{2}=4$ for image 2 respectively, and selected the Gaussian kernel($\sigma=3$) for grayscale level set segmentation. Fig. \ref{4.0.1} and fig. \ref{4.0.2} illustrate the segmentation results of the proposed model. Even with strong noise influence, the method successfully separates the foreground from the background and accurately extracts the target boundary. We not only use the dice value as a reference, but also introduce the length and curvature to evaluate the segmentation effect. The specific performance indicators are shown in the tab \ref{tab:segmentation_metrics}.

\begin{figure}[htbp]
	\centering
	\begin{subfigure}{0.22\textwidth}
		\centering
		\includegraphics[scale=0.35]{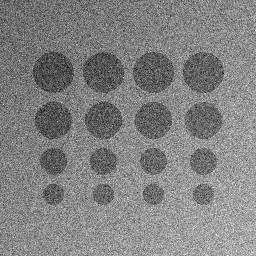}
		\caption{Degraded}
	\end{subfigure}
	\begin{subfigure}{0.22\textwidth}
		\centering
		\includegraphics[scale=0.35]{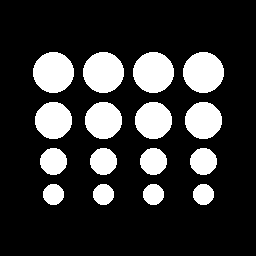}
		\caption{Ground Truth}
	\end{subfigure}
	\begin{subfigure}{0.22\textwidth}
		\centering
		\includegraphics[scale=0.35]{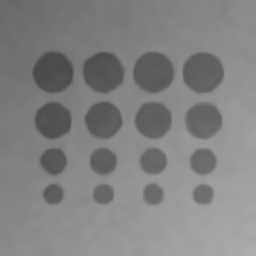}
		\caption{Processed}
	\end{subfigure}
	\begin{subfigure}{0.22\textwidth}
		\centering
		\includegraphics[scale=0.35]{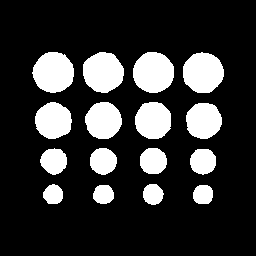}
		\caption{Result}
	\end{subfigure}
	\caption{Results of Image 1}
	\label{4.0.1}
\end{figure}

\begin{figure}[htbp]
	\centering
	\begin{subfigure}{0.22\textwidth}
		\centering
		\includegraphics[width=0.9\textwidth]{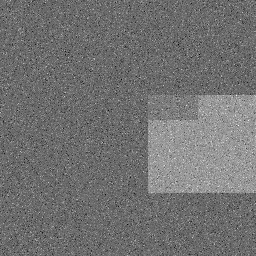}
		\caption{Degraded}
	\end{subfigure}
	\centering
	\begin{subfigure}{0.22\textwidth}
		\centering
		\includegraphics[width=0.9\textwidth]{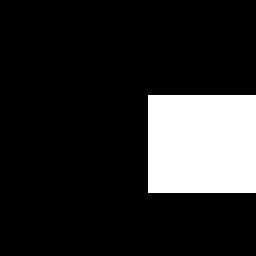}
		\caption{Ground Truth}
	\end{subfigure}
	\centering
	\begin{subfigure}{0.22\textwidth}
		\centering
		\includegraphics[width=0.9\textwidth]{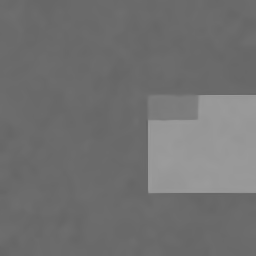}
		\caption{Processed}
	\end{subfigure}
	\centering
	\begin{subfigure}{0.22\textwidth}
		\centering
		\includegraphics[width=0.9\textwidth]{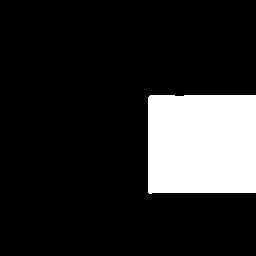}
		\caption{Result}
	\end{subfigure}
	\caption{Results of Image 2}
	\label{4.0.2}
\end{figure}

\begin{table}
	\centering
	\caption{Segmentation Performance Metrics}
	\label{tab:segmentation_metrics}
	\begin{tabular}{ccccc}
		\hline
		\textbf{Method} & \textbf{Dice} & \textbf{IoU} & \textbf{Length} & \textbf{MAC}\\
		\hline
		Image 1           & 0.9799        & 0.9606         & 3311.11    & 0.2905 \\
		Image 2     & 0.9992        & 0.9985          & 623.80 & 0.0296     \\
		\hline
	\end{tabular}
\end{table}
Fig. \ref{4.0.3} (a-d) show the bias correction result and its segmentation on image 3. For comparison, the classical RSF model is applied to the same inhomogeneous image under different initial contour settings. As shown in fig. \ref{4.0.3} (e-h), the RSF model exhibits strong sensitivity to initialization: different initial contours lead to distinctly different segmentation results. Using gray level as a constraint to select the subspace of level set function for energy calculation can avoid the impact of initialization on segmentation results, which is also the advantage of gray level set.

\begin{figure}[htbp]
	\centering
	\begin{subfigure}{0.22\textwidth}
		\centering
		\includegraphics[scale=0.7]{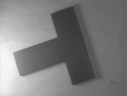}
		\caption{Original}
	\end{subfigure}
	\begin{subfigure}{0.22\textwidth}
		\centering
		\includegraphics[scale=0.7]{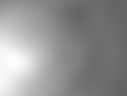}
		\caption{bias}
	\end{subfigure}
	\begin{subfigure}{0.22\textwidth}
		\centering
		\includegraphics[scale=0.7]{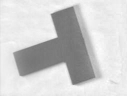}
		\caption{corrected}
	\end{subfigure}
	\begin{subfigure}{0.22\textwidth}
		\centering
		\includegraphics[scale=0.7]{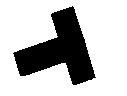}
		\caption{result}
	\end{subfigure}
	\begin{subfigure}{0.22\textwidth}
		\centering
		\includegraphics[scale=1]{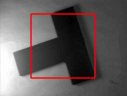}
		\caption{Init1}
	\end{subfigure}
	\begin{subfigure}{0.22\textwidth}
		\centering
		\includegraphics[scale=1]{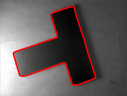}
		\caption{Case1}
	\end{subfigure}
	\begin{subfigure}{0.22\textwidth}
		\centering
		\includegraphics[scale=1]{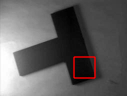}
		\caption{Init2}
	\end{subfigure}
	\begin{subfigure}{0.22\textwidth}
		\centering
		\includegraphics[scale=1]{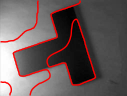}
		\caption{Case2}
	\end{subfigure}
	\caption{Results of Image 3}
	\label{4.0.3}
\end{figure}

\subsection{Validation on Real Images}

To further evaluate the practical applicability of the proposed grayscale level set method, we conduct experiments on several types of real-world images: white blood cell (WBC) images from the publicly available blood cell segmentation dataset, T1-weighted brain MRI volumes from the BrainWeb database, vessel images and crack images. These data cover a wide range of degradation characteristics, including intensity inhomogeneity, low contrast, weak boundaries, and structured noise, which are common challenges in medical and industrial imaging.

\begin{table*}
	\centering
	\caption{WBC dataset Segmentation Performance Metrics}
	\label{tab:wbc dataset}
	\begin{tabular}{cccccc}
		\hline
		\textbf{Image} & \textbf{Dice} & \textbf{IoU} & \textbf{HD95} & \textbf{Length} & \textbf{MAC}\\
		\hline
		029             & 0.9882        & 0.9767       & 1.0000    & 750.22 & 0.3037     \\
		031             & 0.9583        & 0.9200       & 2.2360   & 540.36 & 0.3375     \\
		081             & 0.9795        & 0.9597       &  1.4142    & 696.62 & 0.3526    \\
		230             & 0.9884        & 0.9770       & 1.0000   & 885.79 & 0.3473   \\
		285             & 0.8850        & 0.7936       &  6.3560    & 700.28 & 0.3148  \\
		290             & 0.9245        & 0.8596       &  4.2426    & 758.61 & 0.1617    \\
		\hline
	\end{tabular}
\end{table*}

For WBC dataset, a total of 6 representative samples (indexed 029, 031, 081, 230, 285, and 290 in this analysis) are selected to show model performance. Fig. \ref{WBC} illustrates the segmentation performance. The proposed method is still effective in low contrast noisy images. Tab. \ref{tab:wbc dataset} shows the details of metrics, including Dice, IoU, Accuracy, F1 and other geometric features. For BrainWeb dataset, we select the slice with 20\% intensity non-uniformity and 3\% noise for the experiment. Fig. \ref{BrainWeb} and tab. \ref{tab:brainweb} show the partial results of segmentation and its metrics. Besides, we choose the image in the DIAS dataset to segment the intracranial arteries in fig. \ref{DIAS}.

\begin{figure}[htbp]
	\centering
	\setlength{\tabcolsep}{3pt} 
	\begin{tabular}{c*{6}{c}} 
		\multicolumn{1}{c}{\textbf{}} & 
		\multicolumn{1}{c}{\textbf{029}} & 
		\multicolumn{1}{c}{\textbf{031}} & 
		\multicolumn{1}{c}{\textbf{081}} & 
		\multicolumn{1}{c}{\textbf{230}} & 
		\multicolumn{1}{c}{\textbf{285}} & 
		\multicolumn{1}{c}{\textbf{290}} \\
		\rotatebox{90}{\textbf{Original}} &
		\includegraphics[width=0.12\textwidth]{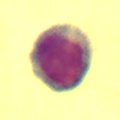} &
		\includegraphics[width=0.12\textwidth]{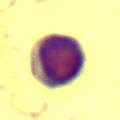} &
		\includegraphics[width=0.12\textwidth]{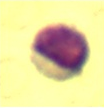} &
		\includegraphics[width=0.12\textwidth]{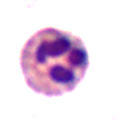} &
		\includegraphics[width=0.12\textwidth]{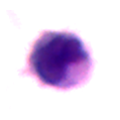} &
		\includegraphics[width=0.12\textwidth]{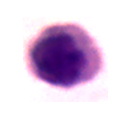} \\
		\rotatebox{90}{\textbf{\ \ \ GT}} &
		\includegraphics[width=0.12\textwidth]{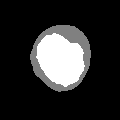} &
		\includegraphics[width=0.12\textwidth]{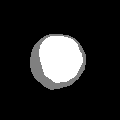} &
		\includegraphics[width=0.12\textwidth]{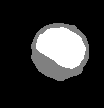} &
		\includegraphics[width=0.12\textwidth]{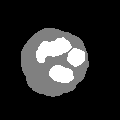} &
		\includegraphics[width=0.12\textwidth]{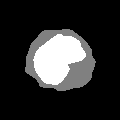} &
		\includegraphics[width=0.12\textwidth]{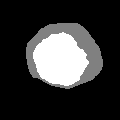} \\
		\rotatebox{90}{\textbf{Smooth}} &
		\includegraphics[width=0.12\textwidth]{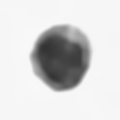} &
		\includegraphics[width=0.12\textwidth]{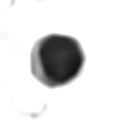} &
		\includegraphics[width=0.12\textwidth]{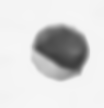} &
		\includegraphics[width=0.12\textwidth]{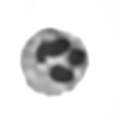} &
		\includegraphics[width=0.12\textwidth]{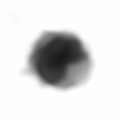} &
		\includegraphics[width=0.12\textwidth]{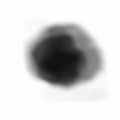} \\
		\rotatebox{90}{\textbf{Results}} &
		\includegraphics[width=0.12\textwidth]{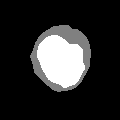} &
		\includegraphics[width=0.12\textwidth]{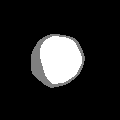} &
		\includegraphics[width=0.12\textwidth]{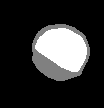} &
		\includegraphics[width=0.12\textwidth]{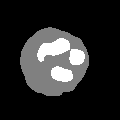} &
		\includegraphics[width=0.12\textwidth]{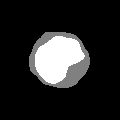} &
		\includegraphics[width=0.12\textwidth]{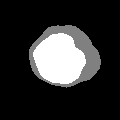} \\
	\end{tabular}
	\caption{Results of WBC dataset}
	\label{WBC}
\end{figure}

\begin{figure}[htbp]
	\centering
	\setlength{\tabcolsep}{3pt} 
	\begin{tabular}{c c c c c c} 
		\multicolumn{1}{c}{\textbf{}} & 
		\multicolumn{1}{c}{\textbf{060}} & 
		\multicolumn{1}{c}{\textbf{063}} & 
		\multicolumn{1}{c}{\textbf{086}} & 
		\multicolumn{1}{c}{\textbf{089}} & 
		\multicolumn{1}{c}{\textbf{098}} \\ 
		\rotatebox{90}{\textbf{\ \ \ Original}} &
		\includegraphics[width=0.15\textwidth]{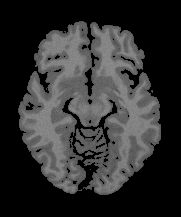} &
		\includegraphics[width=0.15\textwidth]{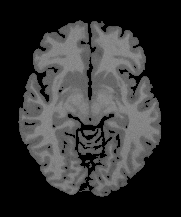} &
		\includegraphics[width=0.15\textwidth]{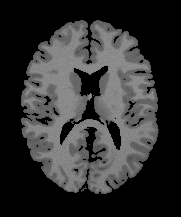} &
		\includegraphics[width=0.15\textwidth]{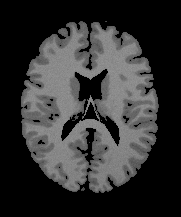} &
		\includegraphics[width=0.15\textwidth]{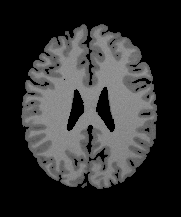} \\
		\rotatebox{90}{\textbf{\ \ \ \ \ GT}} &
		\includegraphics[width=0.15\textwidth]{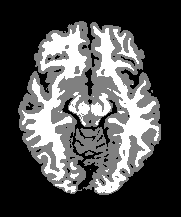} &
		\includegraphics[width=0.15\textwidth]{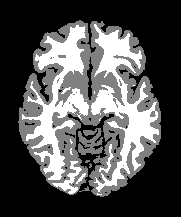} &
		\includegraphics[width=0.15\textwidth]{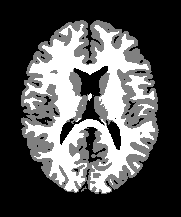} &
		\includegraphics[width=0.15\textwidth]{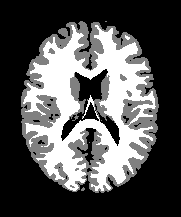} &
		\includegraphics[width=0.15\textwidth]{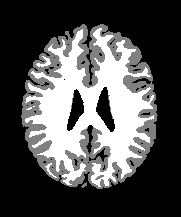} \\
		\rotatebox{90}{\textbf{\ \ Smooth}} &
		\includegraphics[width=0.15\textwidth]{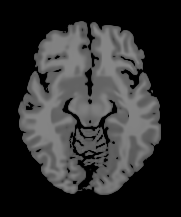} &
		\includegraphics[width=0.15\textwidth]{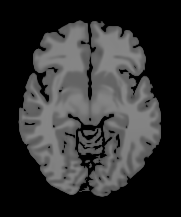} &
		\includegraphics[width=0.15\textwidth]{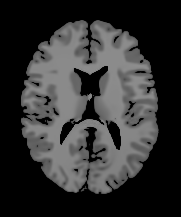} &
		\includegraphics[width=0.15\textwidth]{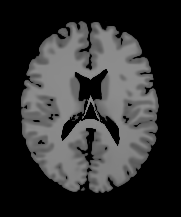} &
		\includegraphics[width=0.15\textwidth]{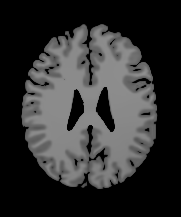} \\
		\rotatebox{90}{\textbf{\ \ \ Results}} &
		\includegraphics[width=0.15\textwidth]{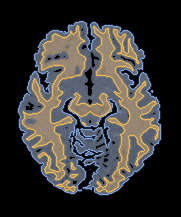} &
		\includegraphics[width=0.15\textwidth]{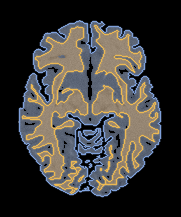} &
		\includegraphics[width=0.15\textwidth]{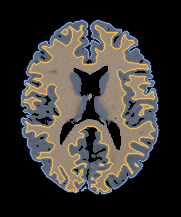} &
		\includegraphics[width=0.15\textwidth]{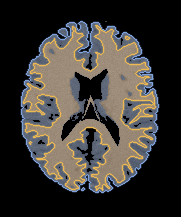} &
		\includegraphics[width=0.15\textwidth]{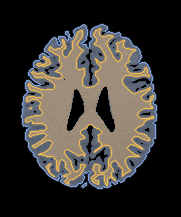} \\
	\end{tabular}
	\caption{Results of BrainWeb dataset}
	\label{BrainWeb}
\end{figure}

\begin{table}
	\centering
	\caption{BrainWeb dataset Segmentation Performance Metrics}
	\label{tab:brainweb}
	\begin{tabular}{ccccc} 
		\hline
		\multirow{2}{*}{\textbf{Image}} 
		& \multicolumn{2}{c}{\textbf{Gray Matter}}   
		& \multicolumn{2}{c}{\textbf{White Matter}}  
		\\
		& \textbf{Dice} & \textbf{IoU} & \textbf{Dice} & \textbf{IoU} 
		\\
		\hline
		060 & 0.9593 & 0.9218 & 0.9398 & 0.8864 \\
		063 & 0.9561 & 0.9159 & 0.9420 & 0.8903 \\
		086 & 0.9509 & 0.9064 & 0.9554 & 0.9145 \\
		089 & 0.9469 & 0.8991 & 0.9606 & 0.9243 \\
		098 & 0.9609 & 0.9247 & 0.9719 & 0.9453 \\
		\hline
	\end{tabular}
\end{table}

\begin{figure}[htbp]
	\centering
	\begin{subfigure}{0.15\textwidth}
		\centering
		\includegraphics[width=\textwidth]{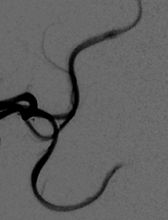}
		\captionsetup{font=tiny}
		\caption{Degraded}
	\end{subfigure}
	\begin{subfigure}{0.15\textwidth}
		\centering
		\includegraphics[width=\textwidth]{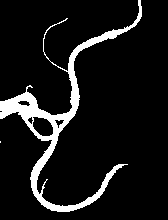}
		\captionsetup{font=tiny}
		\caption{GT}
	\end{subfigure}
	\begin{subfigure}{0.15\textwidth}
		\centering
		\includegraphics[width=\textwidth]{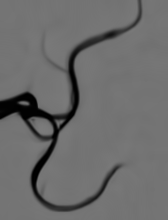}
		\captionsetup{font=tiny}
		\caption{Processed}
	\end{subfigure}
	\begin{subfigure}{0.15\textwidth}
		\centering
		\includegraphics[width=\textwidth]{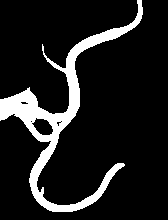}
		\captionsetup{font=tiny}
		\caption{Result}
	\end{subfigure}
	\caption{Results of DIAS dataset}
	\label{DIAS}
\end{figure}

These results demonstrate that the proposed grayscale level set method consistently achieves high segmentation accuracy across diverse real-world imaging modalities. The performance is particularly notable on vessel datasets, where thin and low contrast structures are reliably extracted without any length regularization. 

\subsection{Comparison with Variation Joint Model}
For images with strong multiplicative gamma noise and intensity inhomogeneity, Zhou \cite{zhou_seg} proposed a variational joint model, whose corresponding variational denoising module is as follows:
\begin{equation}
	\inf_{U} \lambda \int_{\Omega} \left( \log u + \frac{\varphi(\tilde{f})}{u} \right) dx + \int_{\Omega} \alpha(x) \, |Du|
	\label{Zhou denoising}
\end{equation}
where $U=\{u \in BV(\Omega) \mid u > 0\}$, $\varphi(\cdot)$ is a nonlinear function and $\alpha(x)$ depends on the enhanced image.

For a fair comparison, we integrate the denoising module of the variational joint model into the proposed grayscale level set framework while keeping the denoising parameters unchanged. Fig.~\ref{4.2.1} presents the segmentation results of the two methods. Although slight differences can be observed in the denoising results, both methods accurately delineate the target region and preserve the principal contour structures, demonstrating comparable segmentation accuracy.

From the perspective of the optimization mechanism, the variational joint model performs collaborative optimization of the denoising and segmentation modules, resulting in a strong coupling between the two tasks. Since image denoising and image segmentation correspond to different optimization objectives, their interactions are difficult to characterize theoretically. Consequently, jointly optimizing these two modules does not necessarily lead to further improvements in segmentation performance. In contrast, the proposed grayscale level set framework completely decouples preprocessing from segmentation, providing a more flexible and modular optimization strategy.

The primary advantage of the proposed framework lies in its computational efficiency. Compared with the traditional variational joint model, the proposed method substantially reduces the computational cost while maintaining comparable segmentation accuracy. As the image resolution increases, the computational burden of the joint model grows rapidly, whereas the proposed incremental update strategy naturally supports parallel implementation and significantly accelerates the threshold optimization process. The quantitative comparisons in Tab.~\ref{tab:speed_comparison} further demonstrate the efficiency improvement achieved by the proposed implementation across different image resolutions.

\begin{table*}
	\centering
	\caption{CPU time}
	\label{tab:speed_comparison}
	\begin{tabular}{c|cccc}
		\hline
		\textbf{Image Size} & \textbf{Joint Model} & \textbf{Denoising} & \textbf{Segmentation} & \textbf{Algorithm \ref{alg:incremental update}}\\ 
		\hline
		100$\times$100   & 28.8018           & 2.0182  &0.1804 &\textbf{0.0306}  \\
		256$\times$256 & 46.8250           & 4.4410 & 0.5572 & \textbf{0.1052}  \\
		512$\times$512 & 116.0726            & 20.706  & 4.4265 & \textbf{0.3433}  \\
		\hline
	\end{tabular}
\end{table*}

\begin{figure}[htbp]
	\centering
	\begin{subfigure}{0.21\textwidth}
		\centering
		\includegraphics[width=0.8\textwidth]{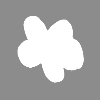}
		\caption{Original}
	\end{subfigure}
	\begin{subfigure}{0.21\textwidth}
		\centering
		\includegraphics[width=0.8\textwidth]{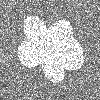}
		\caption{Noise}
	\end{subfigure}
	\begin{subfigure}{0.21\textwidth}
		\centering
		\includegraphics[width=0.8\textwidth]{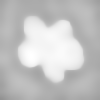}
		\caption{Denoised}
	\end{subfigure}
	\begin{subfigure}{0.21\textwidth}
		\centering
		\includegraphics[width=0.8\textwidth]{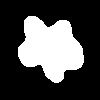}
		\caption{Joint model}
	\end{subfigure}
	\qquad
	\begin{subfigure}{0.21\textwidth}
		\centering
		\includegraphics[width=0.8\textwidth]{4.1pic/zym_flower.png}
		\caption{Original}
	\end{subfigure}
	\begin{subfigure}{0.21\textwidth}
		\centering
		\includegraphics[width=0.8\textwidth]{4.1pic/zym_noise.png}
		\caption{Noise}
	\end{subfigure}
	\begin{subfigure}{0.21\textwidth}
		\centering
		\includegraphics[width=0.8\textwidth]{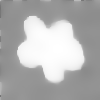}
		\caption{Denoised}
	\end{subfigure}
	\begin{subfigure}{0.21\textwidth}
		\centering
		\includegraphics[width=0.8\textwidth]{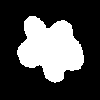}
		\caption{Ours}
	\end{subfigure}
	\caption{Comparison results.}
	\label{4.2.1}
\end{figure}


\subsection{Comparison with U-Net on MoNuSeg Dataset}
To further validate the geometric advantages of proposed grayscale level set method in medical image segmentation, we present a comparative experiment with the classical deep learning segmentation nerwork U-Net on the publicly available Multi-Organ Nucleus Segmentation (MoNuSeg) dataset. The MoNuSeg dataset consists of H\&E-stained histopathology images from various tissue sources, characterized by densely packed nuclei, blurred boundaries, and staining intensity inhomogeneity. 

\begin{figure}[htbp]
	\centering
	\begin{subfigure}{0.16\textwidth}
		\centering
		\includegraphics[scale=0.25]{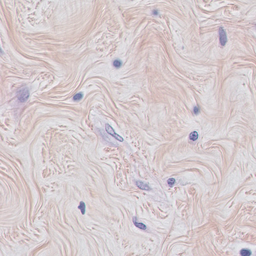}
		\caption{Original}
	\end{subfigure}
	\begin{subfigure}{0.16\textwidth}
		\centering
		\includegraphics[scale=0.25]{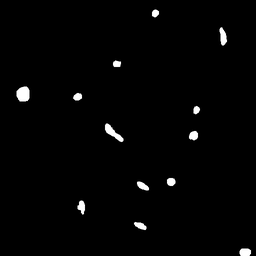}
		\caption{GT}
	\end{subfigure}
	\begin{subfigure}{0.16\textwidth}
		\centering
		\includegraphics[scale=0.25]{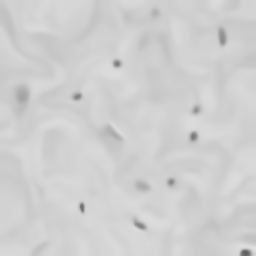}
		\caption{Processed}
	\end{subfigure}
	\begin{subfigure}{0.16\textwidth}
		\centering
		\includegraphics[scale=0.25]{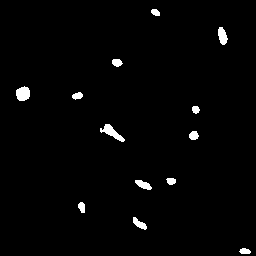}
		\caption{U-Net}
	\end{subfigure}
	\begin{subfigure}{0.16\textwidth}
		\centering
		\includegraphics[scale=0.25]{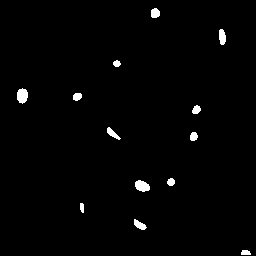}
		\caption{Ours}
	\end{subfigure}
	\caption{Comparison with Unet}
	\label{Comparison with Unet}
\end{figure}

\begin{table}
	\centering
	\caption{Segmentation Performance Metrics}
	\label{tab:Comparison with Unet}
	\begin{tabular}{cccc}
		\hline
		\textbf{Method} & \textbf{Dice} & \textbf{Length} & \textbf{MAC}\\
		\hline
		U-Net & 0.7076  & +94 & 0.3445 \\
		Ours & 0.6422  &\textbf{-8} & \textbf{0.3357}\\
		\hline
	\end{tabular}
\end{table}

Due to the powerful semantic feature extraction ability of deep network, U-Net has a good performance in the traditional indicators based on regional overlap such as Dice. 
Although U-Net achieves better overlop-based metrics, our method geometrically smoother counters and requires no training data.

\subsection{The Robustness of The Proposed Method}

As the high-frequency oscillation part of the image,residual noise destroys the gray distribution characteristics of the image, resulting in the degradation of the performance of the traditional segmentation methods. In practical applications, because the specific information of image noise can not be accurately obtained, blind denoising is often difficult to achieve the optimal effect, and there are inevitably three situations: \textbf{under-denoising(U)}, \textbf{appropriate-denoising(N)} and \textbf{over-denoising(O)}.

In this part, different denoisers are selected to adapt to different noise types. For additive Gaussian noise, traditional filter denoiser BM3D and deep learning denoiser DnCNN \cite{2016Beyond} are selected as mainstream traditional and network denoisers; For multiplicative gamma noise, we use the variational denoiser mentioned above (\ref{Zhou denoising}). Then, we adjust the intensity of denoising appropriately to simulate the blind denoising. 

\begin{table}
	\centering
	\caption{Segmentation Performance Metrics}
	\label{tab: anti-noise}
	\begin{tabular}{ccc}
		\hline
		\textbf{Noise Type} & \textbf{Method} & \textbf{Dice} \\
		\hline
		Additive & BM3D  & 0.9470$\pm$\textbf{0.0047} \\
		Additive & DnCNN  & 0.9461$\pm$\textbf{0.0019}\\
		Multiplicative & Zhou\cite{zhou2025variational} & 0.9831$\pm$\textbf{0.0027}\\
		\hline
	\end{tabular}
\end{table}

Tab \ref{tab: anti-noise}  quantitatively show that the proposed gray level set method can accurately segment the target region no matter for noise type and the segmentation accuracy has little difference for each denoising scene. This characteristic proves that the proposed method is indenpendent on accurate denoising pretreatment, and can resist noise interference. Fig. \ref{4.4.1}, \ref{4.4.2} illustrates the results of experiments.

\begin{figure}[htbp]
	\centering
	\begin{subfigure}{0.16\textwidth}
		\centering
		\includegraphics[width=0.8\textwidth]{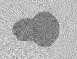}
		\caption{BM3D+U}
	\end{subfigure}
	\begin{subfigure}{0.16\textwidth}
		\centering
		\includegraphics[width=0.8\textwidth]{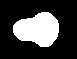}
		\caption{U}
	\end{subfigure}
	\begin{subfigure}{0.16\textwidth}
		\centering
		\includegraphics[width=0.8\textwidth]{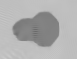}
		\caption{BM3D+N}
	\end{subfigure}
	\begin{subfigure}{0.16\textwidth}
		\centering
		\includegraphics[width=0.8\textwidth]{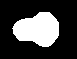}
		\caption{N}
	\end{subfigure}
	\begin{subfigure}{0.16\textwidth}
		\centering
		\includegraphics[width=0.8\textwidth]{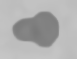}
		\caption{BM3D+O}
	\end{subfigure}
	\begin{subfigure}{0.16\textwidth}
		\centering
		\includegraphics[width=0.8\textwidth]{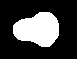}
		\caption{O}
	\end{subfigure}
	
	\begin{subfigure}{0.16\textwidth}
		\centering
		\includegraphics[width=0.8\textwidth]{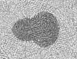}
		\caption{DnCNN+U}
	\end{subfigure}
	\begin{subfigure}{0.16\textwidth}
		\centering
		\includegraphics[width=0.8\textwidth]{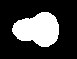}
		\caption{U}
	\end{subfigure}
	\begin{subfigure}{0.16\textwidth}
		\centering
		\includegraphics[width=0.8\textwidth]{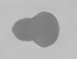}
		\caption{DnCNN+N}
	\end{subfigure}
	\begin{subfigure}{0.16\textwidth}
		\centering
		\includegraphics[width=0.8\textwidth]{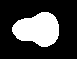}
		\caption{N}
	\end{subfigure}
	\begin{subfigure}{0.16\textwidth}
		\centering
		\includegraphics[width=0.8\textwidth]{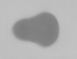}
		\caption{DnCNN+O}
	\end{subfigure}
	\begin{subfigure}{0.16\textwidth}
		\centering
		\includegraphics[width=0.8\textwidth]{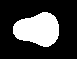}
		\caption{O}
	\end{subfigure}
	\caption{Additive Gaussian Noise}
	\label{4.4.1}
\end{figure}

\begin{figure}[htbp]
	\centering
	\begin{subfigure}{0.16\textwidth}
		\centering
		\includegraphics[width=0.8\textwidth]{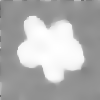}
		\caption{Zhou+U}
	\end{subfigure}
	\begin{subfigure}{0.16\textwidth}
		\centering
		\includegraphics[width=0.8\textwidth]{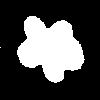}
		\caption{U}
	\end{subfigure}
	\begin{subfigure}{0.16\textwidth}
		\centering
		\includegraphics[width=0.8\textwidth]{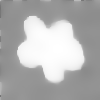}
		\caption{Zhou+N}
	\end{subfigure}
	\begin{subfigure}{0.16\textwidth}
		\centering
		\includegraphics[width=0.8\textwidth]{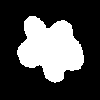}
		\caption{N}
	\end{subfigure}
	\begin{subfigure}{0.16\textwidth}
		\centering
		\includegraphics[width=0.8\textwidth]{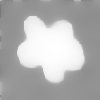}
		\caption{Zhou+O}
	\end{subfigure}
	\begin{subfigure}{0.16\textwidth}
		\centering
		\includegraphics[width=0.8\textwidth]{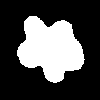}
		\caption{O}
	\end{subfigure}
	\caption{Multiplicative Gamma Noise}
	\label{4.4.2}
\end{figure}

\section{Conclusion}\label{sec5}

In this paper, we proposed a grayscale level set framework for fast image segmentation of multi-degraded images. Motivated by the observation that the data-fitting term alone is sufficient for segmentation on suitably smooth images, the proposed framework eliminates explicit geometric regularization and reformulates conventional level set evolution as a discrete threshold optimization problem. To improve computational efficiency, an incremental update strategy was further developed, reducing the overall computational complexity from $\mathcal{O}(MNS)$ to $\mathcal{O}(NS)$. The proposed framework provides a flexible and modular segmentation paradigm by decoupling image preprocessing from the segmentation stage. Under suitable smoothness assumptions, the grayscale level set offers a reduced search space for optimization while preserving the desired segmentation solution. Extensive experiments on both synthetic and real-world datasets demonstrate that the proposed method achieves competitive segmentation accuracy, significantly improved computational efficiency, and strong robustness under various image degradation conditions. 

Future work will focus on extending the proposed framework to multi-phase segmentation and more general imaging modalities, as well as developing adaptive preprocessing strategies to further improve robustness and reduce parameter tuning.

\section*{Acknowledgements}
This work is partially supported by National Natural Science Foundation of China (U21B2075, 12171123, 12271130) and the Fundamental Research Funds for the Central Universities (2022FRFK-060020).

\bibliographystyle{unsrtnat}
\bibliography{references.bib}  

@article{OSHER198812,
title = {Fronts propagating with curvature-dependent speed: Algorithms based on Hamilton-Jacobi formulations},
journal = {Journal of Computational Physics},
volume = {79},
number = {1},
pages = {12-49},
year = {1988},
issn = {0021-9991},
doi = {https://doi.org/10.1016/0021-9991(88)90002-2},
url = {https://www.sciencedirect.com/science/article/pii/0021999188900022},
author = {Stanley Osher and James A Sethian}
}

@article{1988Snakes,
  title={Snakes: Active Contour Models},
  author={ Kass, Michael  and  Witkin, Andrew  and  Terzopoulos, Demetri },
  journal={International Journal of Computer Vision},
  volume={1},
  number={4},
  pages={321-331},
  year={1988},
}

@INPROCEEDINGS{466871,
  author={Caselles, V. and Kimmel, R. and Sapiro, G.},
  booktitle={Proceedings of IEEE International Conference on Computer Vision}, 
  title={Geodesic active contours}, 
  year={1995},
  volume={},
  number={},
  pages={694-699},
  doi={10.1109/ICCV.1995.466871}}

@article{1989Optimal,
  title={Optimal approximations by piecewise smooth functions and associated variational problems},
  author={ Mumford, David  and  Shah, Jayant },
  journal={Communications on Pure \& Applied Mathematics},
  volume={42},
  number={5},
  pages={577-685},
  year={1989},
}

@ARTICLE{902291,
  author={Chan, T.F. and Vese, L.A.},
  journal={IEEE Transactions on Image Processing}, 
  title={Active contours without edges}, 
  year={2001},
  volume={10},
  number={2},
  pages={266-277},
  doi={10.1109/83.902291}}

@ARTICLE{5557813,
  author={Li, Chunming and Xu, Chenyang and Gui, Changfeng and Fox, Martin D.},
  journal={IEEE Transactions on Image Processing}, 
  title={Distance Regularized Level Set Evolution and Its Application to Image Segmentation}, 
  year={2010},
  volume={19},
  number={12},
  pages={3243-3254},
  doi={10.1109/TIP.2010.2069690}}

@ARTICLE{4623242,
  author={Li, Chunming and Kao, Chiu-Yen and Gore, John C. and Ding, Zhaohua},
  journal={IEEE Transactions on Image Processing}, 
  title={Minimization of Region-Scalable Fitting Energy for Image Segmentation}, 
  year={2008},
  volume={17},
  number={10},
  pages={1940-1949},
  doi={10.1109/TIP.2008.2002304}}

@ARTICLE{9109685,
  author={Yan, Shi and Tai, Xue-Cheng and Liu, Jun and Huang, Hai-Yang},
  journal={IEEE Transactions on Image Processing}, 
  title={Convexity Shape Prior for Level Set-Based Image Segmentation Method}, 
  year={2020},
  volume={29},
  number={},
  pages={7141-7152},
  doi={10.1109/TIP.2020.2998981}}

@article{LI2019123,
title = {A modified level set algorithm based on point distance shape constraint for lesion and organ segmentation},
journal = {Physica Medica},
volume = {57},
pages = {123-136},
year = {2019},
issn = {1120-1797},
doi = {https://doi.org/10.1016/j.ejmp.2018.12.032},
url = {https://www.sciencedirect.com/science/article/pii/S1120179718313681},
author = {Xu Li and Chunming Li and Hairong Liu and Xiaoping Yang}
}

@inproceedings{ronneberger2015u,
	title={U-net: Convolutional networks for biomedical image segmentation},
	author={Ronneberger, Olaf and Fischer, Philipp and Brox, Thomas},
	booktitle={International Conference on Medical image computing and computer-assisted intervention},
	pages={234--241},
	year={2015},
	organization={Springer}
}

@ARTICLE{5754584,
  author={Li, Chunming and Huang, Rui and Ding, Zhaohua and Gatenby, J. Chris and Metaxas, Dimitris N. and Gore, John C.},
  journal={IEEE Transactions on Image Processing}, 
  title={A Level Set Method for Image Segmentation in the Presence of Intensity Inhomogeneities With Application to MRI}, 
  year={2011},
  volume={20},
  number={7},
  pages={2007-2016},
  doi={10.1109/TIP.2011.2146190}}

@article{DBLP:journals/corr/abs-2010-11929,
  author       = {Alexey Dosovitskiy and
                  Lucas Beyer and
                  Alexander Kolesnikov and
                  Dirk Weissenborn and
                  Xiaohua Zhai and
                  Thomas Unterthiner and
                  Mostafa Dehghani and
                  Matthias Minderer and
                  Georg Heigold and
                  Sylvain Gelly and
                  Jakob Uszkoreit and
                  Neil Houlsby},
  title        = {An Image is Worth 16x16 Words: Transformers for Image Recognition
                  at Scale},
  journal      = {CoRR},
  volume       = {abs/2010.11929},
  year         = {2020},
  url          = {https://arxiv.org/abs/2010.11929},
  eprinttype   = {arXiv},
  eprint       = {2010.11929},
  bibsource    = {dblp computer science bibliography, https://dblp.org}
}

@InProceedings{10.1007/978-3-031-25066-8_9,
author="Cao, Hu
and Wang, Yueyue
and Chen, Joy
and Jiang, Dongsheng
and Zhang, Xiaopeng
and Tian, Qi
and Wang, Manning",
editor="Karlinsky, Leonid
and Michaeli, Tomer
and Nishino, Ko",
title="Swin-Unet: Unet-Like Pure Transformer for Medical Image Segmentation",
booktitle="Computer Vision -- ECCV 2022 Workshops",
year="2023",
publisher="Springer Nature Switzerland",
address="Cham",
pages="205--218",
isbn="978-3-031-25066-8"
}

@article{CHEN2024103280,
title = {TransUNet: Rethinking the U-Net architecture design for medical image segmentation through the lens of transformers},
journal = {Medical Image Analysis},
volume = {97},
pages = {103280},
year = {2024},
issn = {1361-8415},
doi = {https://doi.org/10.1016/j.media.2024.103280},
url = {https://www.sciencedirect.com/science/article/pii/S1361841524002056},
author = {Jieneng Chen and Jieru Mei and Xianhang Li and Yongyi Lu and Qihang Yu and Qingyue Wei and Xiangde Luo and Yutong Xie and Ehsan Adeli and Yan Wang and Matthew P. Lungren and Shaoting Zhang and Lei Xing and Le Lu and Alan Yuille and Yuyin}
}

@article{doi:10.1137/24M167411X,
author = {He, Roy Y. and Liu, Hao},
title = {Euler’s Elastica-Based Cartoon-Smooth-Texture Image Decomposition},
journal = {SIAM Journal on Imaging Sciences},
volume = {18},
number = {1},
pages = {526-569},
year = {2025},
doi = {10.1137/24M167411X},
URL = { 
        https://doi.org/10.1137/24M167411X
},
eprint = { 
        https://doi.org/10.1137/24M167411X
}
}

@article{RUDIN1992259,
title = {Nonlinear total variation based noise removal algorithms},
journal = {Physica D: Nonlinear Phenomena},
volume = {60},
number = {1},
pages = {259-268},
year = {1992},
issn = {0167-2789},
doi = {https://doi.org/10.1016/0167-2789(92)90242-F},
url = {https://www.sciencedirect.com/science/article/pii/016727899290242F},
author = {Leonid I. Rudin and Stanley Osher and Emad Fatemi}
}

@article{zhou_seg,
	title={A variational image segmentation model with intensity correction in the presence of high level multiplicative noise},
	author={Zhou, Yamei and Guo, Zhichang and Li, Yao and Wu, Boying},
	journal={Inverse Problems and Imaging},
	volume={19},
	number={5},
	pages={877--902},
	year={2025},
	publisher={Inverse Problems and Imaging}
}

@article{doi:10.1137/20M1310618,
author = {Li, Xu and Yang, Xiaoping and Zeng, Tieyong},
title = {A Three-Stage Variational Image Segmentation Framework Incorporating Intensity Inhomogeneity Information},
journal = {SIAM Journal on Imaging Sciences},
volume = {13},
number = {3},
pages = {1692-1715},
year = {2020},
doi = {10.1137/20M1310618},

URL = { 
    
        https://doi.org/10.1137/20M1310618
    
    

},
eprint = { 
    
        https://doi.org/10.1137/20M1310618
    
    

}
}

@ARTICLE{4271520,
  author={Dabov, Kostadin and Foi, Alessandro and Katkovnik, Vladimir and Egiazarian, Karen},
  journal={IEEE Transactions on Image Processing}, 
  title={Image Denoising by Sparse 3-D Transform-Domain Collaborative Filtering}, 
  year={2007},
  volume={16},
  number={8},
  pages={2080-2095},
  doi={10.1109/TIP.2007.901238}}

@article{WEN2021108049,
title = {Hybrid BM3D and PDE filtering for non-parametric single image denoising},
journal = {Signal Processing},
volume = {184},
pages = {108049},
year = {2021},
issn = {0165-1684},
doi = {https://doi.org/10.1016/j.sigpro.2021.108049},
url = {https://www.sciencedirect.com/science/article/pii/S0165168421000888},
author = {Ying Wen and Zhichang Guo and Wenjuan Yao and Dong Yan and Jiebao Sun}
}

@ARTICLE{5445030,
  author={Tustison, Nicholas J. and Avants, Brian B. and Cook, Philip A. and Zheng, Yuanjie and Egan, Alexander and Yushkevich, Paul A. and Gee, James C.},
  journal={IEEE Transactions on Medical Imaging}, 
  title={N4ITK: Improved N3 Bias Correction}, 
  year={2010},
  volume={29},
  number={6},
  pages={1310-1320},
  doi={10.1109/TMI.2010.2046908}}

@article{2023Segment,
 	title={Segment Anything},
 	author={ Kirillov, Alexander  and  Mintun, Eric  and  Ravi, Nikhila  and  Mao, Hanzi  and  Rolland, Chloe  and  Gustafson, Laura  and  Xiao, Tete  and  Whitehead, Spencer  and  Berg, Alexander C.  and  Lo, Wan Yen },
 	journal={IEEE},
 	year={2023},
 }

@ARTICLE{8341753,
 	author={Ali, Haider and Rada, Lavdie and Badshah, Noor},
 	journal={IEEE Transactions on Image Processing}, 
 	title={Image Segmentation for Intensity Inhomogeneity in Presence of High Noise}, 
 	year={2018},
 	volume={27},
 	number={8},
 	pages={3729-3738},
 	doi={10.1109/TIP.2018.2825101}}

@article{2016Beyond,
 	title={Beyond a Gaussian Denoiser: Residual Learning of Deep CNN for Image Denoising},
 	author={ Zhang, Kai  and  Zuo, Wangmeng  and  Chen, Yunjin  and  Meng, Deyu  and  Zhang, Lei },
 	journal={IEEE Transactions on Image Processing},
 	volume={26},
 	number={7},
 	pages={3142-3155},
 	year={2016},
 }

@article{doi:10.1137/24M1676612,
 	author = {Hsieh, Po-Wen and Tseng, Chung-Lin and Yang, Suh-Yuh},
 	title = {Additive-Bias-Correction Variational Model for Noisy and Intensity-Inhomogeneous Image Segmentation},
 	journal = {SIAM Journal on Imaging Sciences},
 	volume = {18},
 	number = {2},
 	pages = {1235-1259},
 	year = {2025},
 	doi = {10.1137/24M1676612},
 	
 	URL = { 
 	
 	https://doi.org/10.1137/24M1676612
 	
 	
 	
 	},
 	eprint = { 
 	
 	https://doi.org/10.1137/24M1676612
 	
 	
 	
 	}
 }

@article{ADALSTEINSSON1995269,
 	title = {A Fast Level Set Method for Propagating Interfaces},
 	journal = {Journal of Computational Physics},
 	volume = {118},
 	number = {2},
 	pages = {269-277},
 	year = {1995},
 	issn = {0021-9991},
 	doi = {https://doi.org/10.1006/jcph.1995.1098},
 	url = {https://www.sciencedirect.com/science/article/pii/S0021999185710984},
 	author = {David Adalsteinsson and James A. Sethian}
 }

@article{zhou2025variational,
 	title={A variational model to remove multiplicative noise based on SAR image feature preservation},
 	author={Zhou, Yamei and Guo, Zhichang and Li, Yao and Yao, Wenjuan and Wu, Boying},
 	journal={Inverse Problems and Imaging},
 	volume={19},
 	number={2},
 	pages={253--281},
 	year={2025},
 	publisher={Inverse Problems and Imaging}
 }

@article{liu2022deep,
 	title={Deep convolutional neural networks with spatial regularization, volume and star-shape priors for image segmentation},
 	author={Liu, Jun and Wang, Xiangyue and Tai, Xue-Cheng},
 	journal={Journal of Mathematical Imaging and Vision},
 	volume={64},
 	number={6},
 	pages={625--645},
 	year={2022},
 	publisher={Springer}
 }

@article{liu2025convex,
 	title={Convex Shape Prior for Deep Convolution Neural Network-Based Image Segmentation},
 	author={Liu, Jun and Zhang, Kehui and Tai, Xue-Cheng and Luo, Shousheng},
 	journal={Journal of Mathematical Imaging and Vision},
 	volume={67},
 	number={6},
 	pages={61},
 	year={2025},
 	publisher={Springer}
 }






\end{document}